\documentclass[a4paper]{article}

\usepackage{fullpage}
\usepackage{amsfonts, amsmath, amsthm}
\usepackage[inline]{enumitem}
\usepackage{graphicx}

\usepackage{thmtools, thm-restate}

\usepackage{cite} 
\usepackage{booktabs} 
\usepackage{multirow} 
\usepackage{threeparttable} 
\usepackage[hidelinks]{hyperref} %

\usepackage{algorithm}
\usepackage{algpseudocode}
\usepackage{tikz}

\usepackage{authblk}

\algnewcommand\algorithmicto{\textbf{to}}
\algrenewtext{For}[3]%
{\algorithmicfor\ #1 $\gets$ #2 \algorithmicto\ #3 \algorithmicdo}

\algnewcommand\algorithmicforeach{\textbf{for each}}
\algnewcommand\algorithmicin{\textbf{in}}
\algrenewtext{ForAll}[2]%
{\algorithmicforeach\ #1 \algorithmicin\ #2}

\newtheorem{theorem}{Theorem}
\newtheorem{definition}{Definition}
\newtheorem{lemma}{Lemma}
\declaretheorem[name=Proposition]{proposition}

\title{Epistasis-based Basis Estimation Method for Simplifying the Problem Space of an Evolutionary Search in Binary Representation}

\makeatletter
\renewcommand*{\maketitle}{\bgroup\setlength{\parindent}{0pt}
	\begin{flushleft}
		\textbf{\Large \@title}
		\vskip 1em
		\setlength{\leftskip}{2.5 em}
		\@author
	\end{flushleft}\egroup
}
\makeatother
\author{\normalsize \textbf{Junghwan Lee and Yong-Hyuk Kim}}
\affil{\small Department of Computer Science, Kwangwoon University, 20 Kwangwoon-ro, Nowon-gu, Seoul 01897, Republic of Korea}
\affil{ Correspondence should be addressed to Yong-Hyuk Kim; yhdfly@kw.ac.kr}
\date{}

\newcommand*{\thead}[1]{\multicolumn{1}{c}{#1}}
\newcommand*{\multithead}[2]{\multicolumn{#1}{@{}c@{}}{\begin{tabular}{c} #2 \end{tabular}}}

\begin{document}
	\maketitle
	
	\begin{abstract}
		An evolutionary search space can be smoothly transformed via a suitable change of basis; however, it can be difficult to determine an appropriate basis. In this paper, a method is proposed to select an optimum basis can be used to simplify an evolutionary search space in a binary encoding scheme. The basis search method is based on a genetic algorithm and the fitness evaluation is based on the epistasis, which is an indicator of the complexity of a genetic algorithm. Two tests were conducted to validate the proposed method when applied to two different evolutionary search problems. The first searched for an appropriate basis to apply, while the second searched for a solution to the test problem. The results obtained after the identified basis had been applied were compared to those with the original basis, and it was found that the proposed method provided superior results.
	\end{abstract}

\section{Introduction}
Binary encoding typically uses a standard basis, and when a non-standard basis is used, the structure of the problem space may become quite different from that of the original problem. In an evolutionary search, various methods can be used to change a problem space by adjusting the basis, including gene rearrangement, different encoding methods, and the use of an eigen-structure ~\cite{hwang2006multi, CHANG20091210, sankoff1998multiple, raidl2000weighted, falkenauer1994new, gen2006genetic, wang2009genetic, lotfi2013genetic, pernkopf2001feature, wang2006new, wu2012novel, wyatt2003finding}.

An investigation was conducted to elucidate the possibility of changing the basis in binary encoding and the corresponding effects on the genetic algorithm (GA)~\cite{kim2008effect}; however, it was not possible to determine which basis should be applied to smooth the problem search space. In genetics, epistasis means that the phenotypic effect of one gene is masked by another gene; however, in GA, it refers to any type of gene interaction. In a problem with a large epistasis, as the genes are extremely inter-dependent, the fitness landscape of the problem space is very complex and the problem is difficult~\cite{davidor1990epistasis}.
Several studies have been conducted to assess the difficulty of problems from the perspective of epistasis~\cite{seo2003new, seo2004new, ventresca2007epistasis, seo2005computing, reeves1995epistasis, naudts2000comparison, beasley1993reducing}.
Epistasis has the advantage that it is possible to measure the extent of nonlinearity only with fitness function.
In this paper, we define the difficulty of the problem or problem search space as the nonlinearity level of gene expression. Also, we use epistasis as a measure for the difficulty of the problem.

There are three main contributions of this paper. First, an epistasis approximation is used to identify a basis that will reduce the complexity of an evolutionary search problem. Second, the basis is expressed by a variable-length encoding scheme using an elementary matrix. Finally, a GA is defined that can be used to change the basis of an evolutionary search space. This means that when a basis is given, one can tell how it affects the GA.
Our intention in this study is that a non-separable problem can be transformed into a separable problem by performing an appropriate basis transformation. Such an altered environment enables GA to search space effectively.

This paper is organized as follows: Section~\ref{sec2:motivation} describes the principle of reducing the complexity of a problem space in an evolutionary search by changing the basis and presents the motivation for evaluating the basis using the epistasis. In Section~\ref{sec3:cob}, a method is introduced for changing a standard basis to another basis for a binary encoding problem. Then, a GA is introduced that can be used to apply a change of basis. Once an appropriate basis has been selected, this algorithm is more efficient at searching for a solution than the conventional GA. In Section~\ref{sec4:eval_basis}, a method is proposed for estimating a basis that reduces the complexity of an evolutionary search problem. Section~\ref{sec5:find_basis_GA} describes a GA that can be used to search for a basis by applying the proposed estimation method. Here, a variable length encoding scheme that consists of an elementary matrix is employed so to increase the efficiency of the search for an appropriate basis in the problem space. Section~\ref{sec6:experiments} presents a description of the tests used to validate the method and then discusses the results. In the tests, an appropriate basis for the target problem is found via the GA, and then the identified basis is applied to the target problem. The conclusions that can be drawn from this study are presented in Section~\ref{sec7:conclusions}.

\section{Motivation} \label{sec2:motivation}
In this section, the concept of the epistasis is introduced as a means of estimating a basis that will reduce the complexity of the problem. First, a principal component analysis (PCA) is used to extract important information by changing the basis in real number encoding. Next, an example of changing the basis in binary encoding is presented to illustrate that a complex problem can be converted to a simple problem by changing the basis. Lastly, the epistasis between the original and modified problems are compared. If the epistasis of the problem decreases when the basis is changed, it implies the complexity of the original problem has decreased. Thus, a suitable basis can be identified using the changes in the epistasis before and after the prospective basis has been applied to the problem of interest.

\subsection{An Example of Changing a Basis in \boldmath$\mathbb{R}^n$}
A PCA is used to obtain the principal components of the data by transforming the data into a new coordinate system via an orthogonal transformation. When the data is projected in the coordinate system, the position where the variance is the largest becomes the first principal component. The second principal component is in a position that is orthogonal to the previous component at the position with the second largest variance. Consequently, if the eigenvectors and eigenvalues of the covariance matrix are obtained and sorted in descending order, the principal components can be found. This is identical to changing the basis from the original coordinate system to a coordinate system based on the variance of the data. In general, by using only the important principal components, lost data are used. 

\subsection{Change of Basis in Binary Representation} \label{sec22:cob}
Binary encoding typically employs a standard basis; however, it is sometimes easier to manipulate a problem in a non-standard basis. The following example illustrates that the relationship between the basis vectors is dependent on the basis. Here, $ \mathbb{Z}_2 $ is a field that has elements of zero and one, the addition operator corresponds to the exclusive-or (XOR) operator, and the multiplication operator corresponds to the AND operator. The standard basis $ B_s $ for vector space $ \mathbb{Z}_2^n $ is $ \left\{e_1,e_2,\ldots,e_n\right\} $, where $ e_i $ consists of column vectors in which the $ i $-th entry is one and the remaining $ n-1 $ entries are zero.

In the vector space $ \mathbb{Z}_2^n $, if the vector $ v $ and the evaluation function $ F $ are as follows, then the basis vector $ e_i $ of $ B_s $ has a dependency relationship with the other basis vectors $ e_j $ in $ F $.
\begin{align*}
v&=\sum_{i=1}^{n}{\alpha_ie_i}=\left(\alpha_1,\alpha_2\ldots,\alpha_n\right), \\
F\left(v\right)&=\sum_{i=1}^{n}{\left(\alpha_1\oplus\alpha_2\oplus\cdots\oplus\alpha_n\right)\oplus\alpha_i},
\end{align*}
where $ \alpha_i\in\mathbb{Z}_2 $ and $ \oplus $ is the XOR operator.

Let us assume a function $ F^\prime $ performs the same operation as $ F $ but in a new basis and suppose $ n $ is even. If a set $ B $ is composed as follows:
\begin{equation*}
B=\left\{e_i^\prime \,\middle|\, \sum_{e_j\in B_s} e_j-e_i,\, \forall i= 1,2,\ldots,n \right\}.
\end{equation*}
then $ B $ becomes the basis. One property of a basis is that every vector can be represented as a linear combination of basis vectors. That is,
\begin{equation*}
v=\sum_{i=1}^{n}{\alpha_ie_i}=\sum_{i=1}^{n}{\alpha_i^\prime e_i^\prime},
\end{equation*}
where $ \alpha_i^\prime=\sum_{j=1}^{n}\alpha_j+\alpha_i $ and $ \left[v\right]_B=\left(\alpha_1^\prime,\alpha_2^\prime\ldots,\alpha_n^\prime\right) $, which is the representation of $ v $ with respect to the basis $ B $.\\
Here, $ F^\prime $ is a function that evaluates $ \left[v\right]_B $, has the same operation as $ F\left(v\right) $, and satisfies the following relationship:
\begin{equation*}
F\left(v\right)=\sum_{i=1}^{n}{\left(\alpha_1\oplus\alpha_2\oplus\cdots\oplus\alpha_n\right)\oplus\alpha_i}=\sum_{i=1}^{n}\alpha_i^\prime=F^\prime\left(\left[v\right]_B\right) .
\end{equation*}
It can be seen that the basis vector $ e_i^\prime $ of $ B $ is independent of the other basis vectors in $ F^\prime $. In fact, $ F^\prime  $ is identical to the onemax problem that counts the number of ones in a bitstring. Therefore, for a vector in which all $ \alpha_i^\prime $ are set to one, the evaluation value becomes the largest value, and if this vector is transformed with the standard basis, an optimum solution can be obtained. Figure~\ref{fig:rel} shows the relationships of the basis vectors according to the basis with $ n = 6 $ in the graphs.
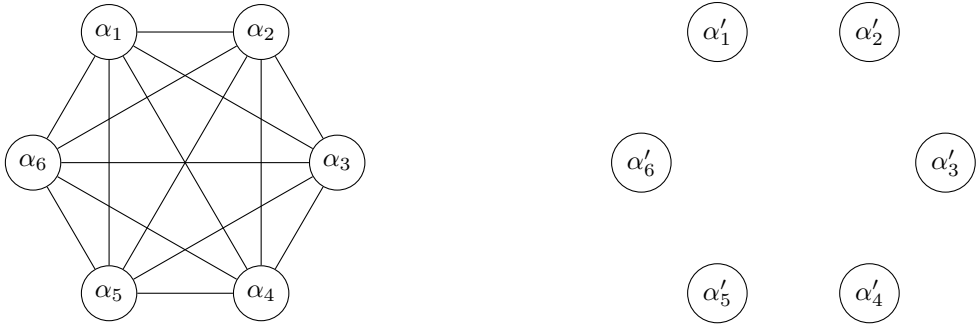
\begin{figure}[H]
	\centering
	\begin{tikzpicture}
	\def \n {6}
	\def \radius {2cm}
	\def \margin {8} 
	
	\foreach \s in {1,...,\n}
	{
		\node(\s)[draw, circle, inner sep=0pt, text width=7mm, align=center,] at ({360/\n * -(\s - 3)}:\radius) {$ \alpha_{\s}$};
	}
	\draw[] (1) -- (2); \draw[] (1) -- (3); \draw[] (1) -- (4); \draw[] (1) -- (5); \draw[] (1) -- (6);
	\draw[] (2) -- (3); \draw[] (2) -- (4); \draw[] (2) -- (5); \draw[] (2) -- (6);
	\draw[] (3) -- (4); \draw[] (3) -- (5); \draw[] (3) -- (6);	
	\draw[] (4) -- (5); \draw[] (4) -- (6);	
	\draw[] (5) -- (6);	
	
	\begin{scope}[shift={(8,0)}]
	\foreach \s in {1,...,\n}
	{
		\node[draw, circle, inner sep=0pt, text width=7mm, align=center,] at ({360/\n * -(\s - 3)}:\radius ) {$ \alpha^\prime_{\s}$};
	}
	\end{scope}
	\end{tikzpicture}
	\caption{Dependency relationship of the different basis vectors of $ n = 6 $ (left side: standard basis, right side: basis $ B $).} \label{fig:rel}
\end{figure}

\subsection{Epistasis According to the Basis}
In a GA, the epistasis indicates the correlation between the genes. If the epistasis for a particular problem is large, then the genes are very inter-dependent, the fitness landscape of the problem space is extremely complex and the problem is difficult. In Section~\ref{sec22:cob}, it was shown that the complexity of a problem varies depending on the basis. The epistasis numerically expresses the complexity of such a problem. In general, when the genes in a problem are very dependent, the epistasis has a large value. In contrast, when the genes are independent, the value is zero.

The results of calculating the epistasis according to the problem size $ n $ of evaluation functions $ F $ and $ F^\prime $ in Section~\ref{sec22:cob} are shown in Table~\ref{tab:epi_func}. In this paper, the method proposed by Davidor~\cite{davidor1990epistasis} is used to compute the epistasis. In $ F $, because the dependency relationship with other basis vectors increases as $ n $ increases, the epistasis also increases. However, for $ F^\prime $, since the basis vectors are independent, the epistasis is zero. Thus, it is expected that the search space can be simplified via an appropriate change of basis.

The epistasis can be used to check if the search space can be simplified by using a particular basis. If the epistasis of the problem after changing the basis is lower than the epistasis of the original problem, then this indicates that the problem has become easier. However, using the epistasis in this way requires all solutions to be searched. An alternative is to estimate the actual epistasis by calculating the epistasis of a sample set of solutions.
Note that nonlinearity may be misleading due to approximation error by solution sampling. It hinders to find the proper basis for the target problem. The target problem may be transformed into a more complex problem through a basis transformation. That is, the basis transformation can rather prevent a GA from efficiently finding the solution.

\begin{table}[H]
	\caption{Epistasies of evaluation functions $ F $ and $ F^\prime $.} \label{tab:epi_func}
	\vspace*{0.2cm}
	\centering
	\begin{tabular}{ccc}
		\toprule
		$ n $ & Epistais in $ F $ & Epistasis in $ F^\prime $   \\ 
		\midrule
		2 &  0.0 & 0   \\
		4 &  1.0 &  0  \\
		6 &  1.5 & 0  \\
		8 &  2.0 & 0  \\
		10&  2.5 & 0  \\
		12&  3.0 & 0  \\
		14&  3.5 & 0  \\
		16&  4.0 & 0\\
		\bottomrule
	\end{tabular}
\end{table}

\section{Change of Basis} \label{sec3:cob}
This section presents a GA that performs an effective search through a change of basis. Before presenting the GA, we introduce the related terminologies and theories of change of basis in binary representation.
Next, we apply the change of basis in the onemax problem to show how the problem actually transformed. In addition, a methodology for evaluating solutions in the transformed problem will be described.
Finally, we propose a GA that effectively searches solutions through applying the change of basis.
On the other hand, searching for an appropriate basis will be covered in Sections~\ref{sec4:eval_basis} and \ref{sec5:find_basis_GA}.

\subsection{Change of Basis in \boldmath$ \mathbb{Z}_2^n $}
A basis for an $ n $-dimensional vector space is a subset that consists of $ n $ vectors and every element of the space can be uniquely represented as a linear combination of basis vectors. Since it is possible to use one or more bases in a vector space, the coordinate representation of a vector with respect to the basis can be transformed via an equivalent representation to other bases via the invertible linear transformation. Such a transformation is called a \textit{change of basis}. The following theorem was derived from the basic theory of linear algebra~\cite{friedberglinear}.
\begin{theorem} \label{thm:2base}                                                                                                                                                                                                                                                                                                                                                                                                                                                                                                                                                               
	Let $ B_1 $ and $ B_2 $ be two bases for $ \mathbb{Z}_2^n $. Then, there exists a nonsingular matrix $ T\in M_{n\times n}\left(\mathbb{Z}_2\right) $ such that for every $ v\in\mathbb{Z}_2^n $, $ T\left[v\right]_{B_1}=\left[v\right]_{B_2} $, where $ \left[v\right]_B $ is the representation of $ v $ with respect to the basis $ B $.
\end{theorem}
A matrix $ A $ is defined as binary if $ A \in M_{n \times n }(\mathbb{Z}_2)$. In general, if $ B $ is the standard basis, $ \left[v\right]_B $ is the representation of $ v $ with respect to the basis $ B $. In Theorem~\ref{thm:2base}, nonsingular binary matrix $ T=\left[T\right]_{B_1}^{B_2} $ is a coordinate-change matrix from basis $ B_1 $ to $ B_2 $. When a $ T $ is given, $ T $ can be viewed as a coordinate-change matrix from the standard basis to $ B_T $, which is related to the $ T $. For every vector $ v\in\mathbb{Z}_2^n $, $ Tv=\left[v\right]_{B_T} $ holds and $ B_T $ is $ \left\{Te_1, Te_2, \ldots,Te_n\right\} $. This study considers a change of basis from a standard basis to another basis. Thus, estimating the basis is equivalent to estimating an appropriate $ T $.

\subsection{Analysis of Changing a Basis in the Onemax Problem} \label{sec32:cob_onemax}
The onemax problem maximizes the number of ones in a bitstring and has zero epistasis. Here, a onemax problem in which the basis was changed using a selected nonsingular binary matrix $ T $ is compared to the original onemax problem. The specific onemax problem of interest has a size of three. The $ T $ is defined as follows:
\begin{equation*}
T=\left(\begin{matrix}1&0&0\\1&0&1\\0&1&0\\\end{matrix}\right).
\end{equation*}
Then, it can be shown that $ B_T=\left\{Te_1,Te_2,Te_3\right\}=\left\{\left(\begin{matrix}1\\1\\0\\\end{matrix}\right),  \left(\begin{matrix}0\\0\\1\\\end{matrix}\right),\left(\begin{matrix}0\\1\\0\\\end{matrix}\right)\right\}$. Table~\ref{tab:onemax} shows the original vector and that obtained using $ Tv=\left[v\right]_{B_T} $. From this, it can be seen that after the basis change, the problem became more complex.

\begin{table}[H]
	\caption{The vectors with a modified basis $ B_T $ and the original vectors in the onemax problem of size 3.} \label{tab:onemax}
	\centering
	\vspace{0.2cm}
	\begin{tabular}{ccc}
		\toprule
		$ v $ & $ \left[v \right]_{B_T} $ & Fitness   \\ 
		\midrule
		$ (1,1,1)^T $ & $ (1,0,1)^T $ & 3   \\ \midrule
		$ (1,1,0)^T $ & $ (1,1,1)^T $ &    \\
		$ (1,0,1)^T $ & $ (1,0,0)^T $ & 2   \\
		$ (0,1,1)^T $ & $ (1,1,1)^T $ &    \\ \midrule
		$ (1,0,0)^T $ & $ (1,1,0)^T $ &    \\
		$ (0,1,0)^T $ & $ (0,1,0)^T $ & 1   \\
		$ (0,0,1)^T $ & $ (0,0,1)^T $ &    \\ \midrule
		$ (0,0,0)^T $ & $ (0,0,0)^T $ & 0   \\
		\bottomrule
	\end{tabular}
\end{table}

The evaluation function $ F $ of the onemax problem is as follows:
\begin{equation*}
F\left(v\right)=\sum_{i=1}^{n}\alpha_i,\ \text{where}\ v=\sum_{i=1}^{n}{\alpha_ie_i}.
\end{equation*}

On the other hand, from Table~\ref{tab:onemax}, it is difficult to identify a rule for the fitness of $ \left[v\right]_{B_T} $ for the onemax problem. The evaluation function $ F^\prime $ of $ \left[v\right]_{B_T} $ can be obtained by computing $ v $ by changing the basis from $ B_T $ to $ B_s $ and evaluating $ v $ with $ F $. That is,
\begin{equation*}
F^\prime\left(\left[v\right]_{B_T}\right)=F\left(T^{-1}\left[v\right]_{B_T}\right)=F\left(v\right),
\end{equation*}
where $ T^{-1} $ is the inverse matrix of $ T $. The above equation is obtained by multiplying the left side by $ T^{-1} $ in $ Tv=\left[v\right]_{B_T} $ and then applying $ F $ to both sides. In this way, the basis on both sides can be easily changed using $ T $ and $ T^{-1} $.

\subsection{Genetic Algorithm with a Change of Basis}
In general, a GA is expected be more efficient when searching for a solution to a simple problem than a complex problem. As shown in Section~\ref{sec22:cob}, a complex problem can be changed to a simple problem by changing the basis. With this in mind, if an appropriate a change of basis is applied to a problem space to be searched by a GA, this will greatly improve the efficiency of the search process. A flowchart of the proposed algorithm is shown in Figure~\ref{fig:flowchart} and the corresponding steps are detailed in Algorithm~\ref{alg:step}.

\begin{algorithm}[H]
	\centering
	\caption{A GA with a change of basis} \label{alg:step}
	\begin{algorithmic}
		\State
		\begin{enumerate}[leftmargin=*, label=Step \arabic*:]
			\item the population $ P $ of the $ GA $ is initialized and the fitness is evaluated.
			\item $ P $ is replaced by the population $ P^\prime $ whereby the standard basis $ B_s $ is changed to the basis $ B $.
			\item by using the genetic operator on the GA, the offspring population $ O^\prime $ is produced from $ P^\prime $.
			\item the fitness of $ O^\prime $ is evaluated using the population $ O $ that was used to change the basis from $ B $ to $ B_s $.
			\item $ P^\prime $ and $ O^\prime $ are used to create a new generation and update $ P^\prime $ to the new generation.
			\item the process from Step 3 onward is repeated as many times as there are generations. When the number of generations has been exceeded, then we return $ P^\prime $ whereby the basis $ B $ is changed to the standard basis $ B_s $.
		\end{enumerate}
	\end{algorithmic}
\end{algorithm}

\begin{figure}[H]
	\centering
	\includegraphics[width=\textwidth]{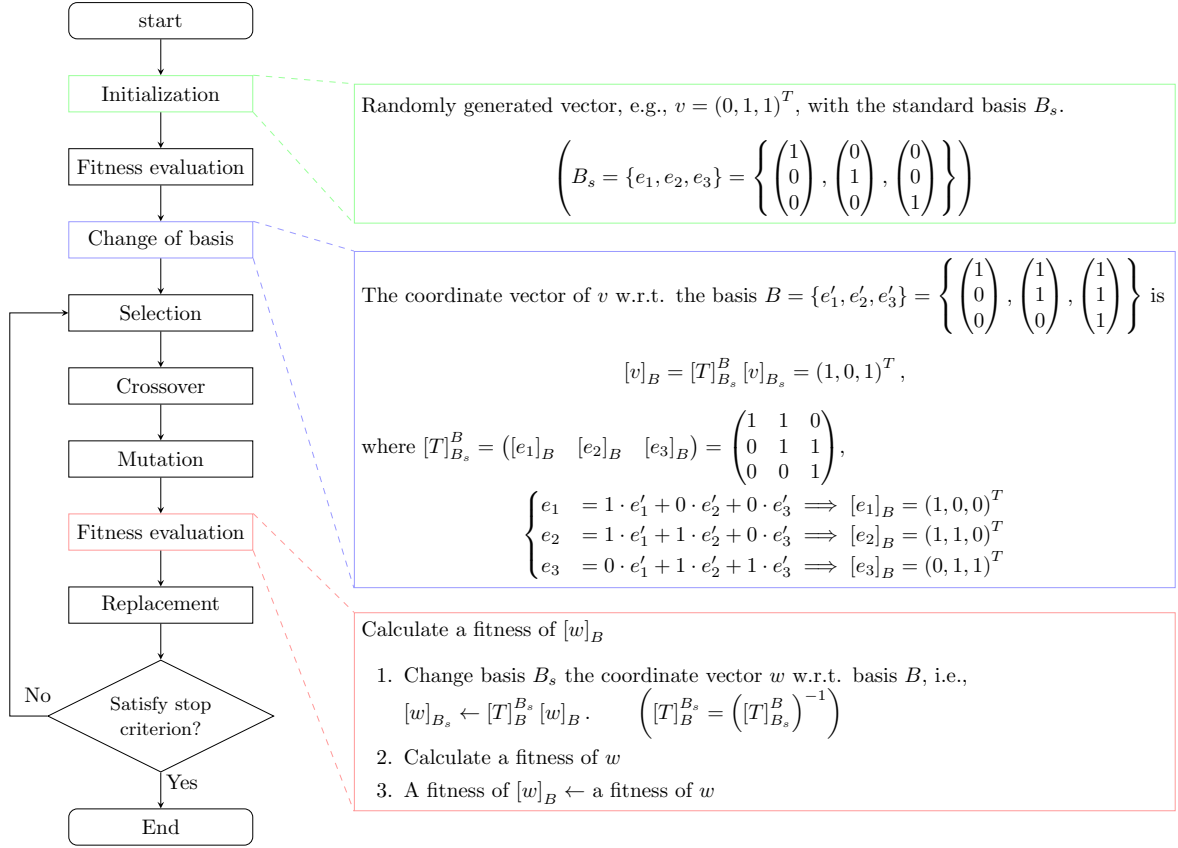}
	\caption{Flowchart of a GA with a change of basis.} \label{fig:flowchart}
\end{figure}

If Steps 2 and 4 are excluded, then Algorithm~\ref{alg:step} produces a typical GA. However, if the problem is transformed with an appropriate basis in Step 2, the original problem space is transformed into an easier problem space, which is expected to make it easier for the GA to find an optimum solution. On the other hand, Step 4 shows that the generated offspring vector is evaluated by changing the basis to the standard basis. This is identical to the method in Section~\ref{sec32:cob_onemax} that evaluates a solution in another basis.

\section{Evaluation of a Basis} \label{sec4:eval_basis}
The objective is to identify a basis that can be used to change a complex problem into to a simple problem. While such a basis was examined in Section~\ref{sec22:cob}, in that case, the change in basis converted the onemax problem from a simple to a complex problem.

When a basis and a target problem are given, a method is proposed that uses the epistasis to evaluate whether the basis is appropriate for the problem space. A meta-genetic algorithm (Meta-GA) is generally used as a method for estimating a hyperparameter of a GA. The two methods are compared to analyze the advantages and disadvantages of the proposed method.

\subsection{Evaluation with Epistasis}
Assume a target problem $ P $ and basis $ B $ are given. To determine the smoothing effect of $ B $ on $ P $, a sampling population $ S $ can be obtained from $ P $. Then, $ S^\prime $ can be obtained by changing the basis for $ S $ from the standard basis to $ B $. The epistasis of $ S^\prime $ that numerically shows the difficulty of the problem can then be calculated. The lower the epistasis is, the more appropriate $ B $ is as a basis for $ P $. The epistasis calculation method proposed by Davidor~\cite{davidor1990epistasis} is shown in Algorithm~\ref{alg:eval_epi}. Suppose the chromosome length is $ l $ and the number of samples in $ S $ is $ s $. Then, the time complexity of evaluating a single basis becomes $ O\left(l^2s\right) $. This is because the cost of executing the change of basis is $ l^2s $. The change of basis is performed for a total of s vectors, and the cost of the change of basis is $ l^2 $ since each vector $ v $ becomes $ \left[v\right]_B $ through $ \left[T\right]_{B_S}^Bv $.

\begin{algorithm}[H]
	\caption{Basis evaluation based on epistasis} \label{alg:eval_epi}
	\begin{algorithmic}[1]
		\Require Sampling population $S$
		\Procedure{Evaluation}{$ B,\,S$} \Comment{Evaluation a basis $ B $}
		\State $S^{\prime} \gets \text{Change of basis from } B_s \text{ to } B \text{ on } S $ \Comment{$ B_s$ is standard basis}
		\State
		
		\ForAll{ind}{$S^\prime$}
		\State $\mu \gets \mu + \text{\Call{v}{ind}/\Call{size}{$S^\prime$}}$\Comment{\Call{v}{ind} is a fitness of ind}
		\For{$i$}{$1$}{\Call{size}{$ind$}}
		\State $ a \gets \text{ind}[i]$ \Comment{$a$ is allele value (0 or 1)}
		\State $A[i][a] \gets A[i][a] +$\Call{v}{ind} \Comment{allele value of $a$}
		\State $C[i][a] \gets C[i][a] + 1$ \Comment{count $A[i][a]$}
		\EndFor
		\EndFor
		\State
		\For{$i$}{$1$}{\Call{size}{$ind$}}
		\ForAll{$a$}{allele values}
		\State $A[i][a] \gets A[i][a] / C[i][a]$
		\State $E[i][a] \gets A[i][a] -\mu$
		\EndFor
		\EndFor
		\State
		
		\ForAll{ind}{$S^\prime$}
		\State $G \gets 0$ \Comment{Genic value}
		\For{$i$}{$1$}{\Call{size}{$ind$}}
		\State $G \gets G + E[i][ind[i]]$
		\EndFor
		\State $G \gets G + \mu$
		\State $\sigma_{S^\prime} \gets \sigma_{S^\prime} + \left(\text{\Call{v}{ind}}-G\right)^2$ \Comment{We have the epistasis $\sigma_{S^\prime}$}
		\EndFor
		\State \textbf{return} $\sigma_{S^\prime}$
		\EndProcedure
	\end{algorithmic}
\end{algorithm}

\subsection{Evaluation with a Meta-genetic Algorithm}
The use of a meta-GA to optimize the parameters and tune GAs was first proposed by Grefenstette~\cite{Grefenstette:1986:OCP:14111.14123}. Here, a meta-GA to determine whether the basis is appropriate for the problem space of the GA. A method of evaluating a basis with a meta-GA is shown in Algorithm~\ref{alg:eval_meta_ga}. By applying Algorithm 1 with a given $ B $ and an instance of GA, $ k $ populations are searched. Then, using the best fitness in each population, the basis is evaluated. That is, when $ k $ units of fitness are found to be acceptable, it is estimated that $ B $ is an appropriate basis of the instance. The reason for searching $ k $ populations is because even with a basis that is not appropriate, a good solution may be obtained by using the GA to search once. To calculate the time complexity, with respect to the target GA, let the number of generations be $ g $, population size $ p $, and chromosome length $ l $. The time cost of line 10 in Algorithm 3 is the largest. When $ p $ offspring are generated, the time consumed is $ pl^2 $. Since this is repeated $ kg $ times, the worst case time complexity becomes $ O(kgpl^2) $. Note that in the experiment evaluated in this paper, $ k $ is set to $ 5 $ and $ g $ is set to the chromosome length.

\begin{algorithm}[H]
	\caption{Basis evaluation in a meta-GA}\label{alg:eval_meta_ga}
	\begin{algorithmic}[1]
		\Require Target GA, Search GA $ k $ times, Generations of GA $ g $
		\Procedure{Evaluation}{$ B,\,GA,\, k, \, g$} \Comment{Evaluation of a basis $ B $}
		\State BestFits[$k$] \Comment{Return array}
		\For{$ i $}{$ 1 $}{$ k $}
		\State $ P \gets \text{GA.InitPopulation} $ \Comment{Initialization of population}
		\State GA.EvalPopulation($P$) \Comment{Evaluation of the population}
		\State $ P \gets \text{ Change of basis from } B_s \text{ to } B \text{ on } P $
		\For{$ j $}{$ 1 $}{$ g $}
		\State $ P^{\prime\prime} \gets \text{GA.Selection}(P^\prime) $
		\State $ O^\prime \gets \text{GA.Recombination}(P^{\prime\prime}) $ \Comment{Perform crossover and mutation operations}
		\State $ O \gets \text{ Change of basis from } B \text{ to } B_s \text{ on } O^\prime $
		\State GA.EvalPopulation($O$)
		\State $ P^\prime \gets \text{ GA.Replace }(O^\prime, P^\prime) $
		\EndFor
		\State BestFits$[i] \gets $ the best fitness of $ P^\prime $
		\EndFor
		\State \textbf{return} BestFits
		\EndProcedure
	\end{algorithmic}
\end{algorithm}

\section{Finding a Basis Using a Genetic Algorithm} \label{sec5:find_basis_GA}
This section describes the components of the GA used to search for a basis for the problem space with the evaluation method outlined in Section~\ref{sec4:eval_basis}. The method of applying a basis and the genetic operator for the encoding are discussed, and the fitness of the basis is evaluated using the method of either Algorithm~\ref{alg:eval_epi} or \ref{alg:eval_meta_ga}.

\subsection{Encoding with an Elementary Matrix}
A nonsingular binary matrix can be regarded as a change from a standard basis to another basis. That is, a basis corresponds to an appropriate the matrix. If a typical 2D type of encoding is used to encode the matrix, a repair mechanism may be required after recombination. In this case, one option is to conduct the repair using the Gauss-Jordan method; however, this will require a length of time equal to $ O\left(n^3\right) $ time.

Every nonsingular matrix can be expressed as a product of elementary matrices~\cite{friedberglinear}. Therefore, in $ GL_n\left(\mathbb{Z}_2\right) $, if a solution is expressed as a product of elementary matrices, it is possible to maintain their invertibility. Each element in an elementary matrix can be expressed by a variable-length linear string~\cite{yoon2014mathematical}, which allows a new encoding to be applied. Note that any recombination method for a variable-length string can be used.
In the following, an elementary row operation is defined and then the elementary matrix in $ M_{n\times n}\left(\mathbb{Z}_2\right) $ is introduced.
\begin{definition}
	Let $ A\in M_{n\times n}\left(\mathbb{Z}_2\right) $. Any one of the following two operations on the rows of $ A $ is called an elementary two operation:
	\begin{enumerate}[label=(\roman*)]
		\item \label{def:ero1} Interchanging any two rows of $ A $, and
		\item \label{def:ero2} Adding a row of $ A $ to another row.
	\end{enumerate}
\end{definition}
Elementary row operations are Type 1 or Type 2 depending on whether they were obtained using \ref{def:ero1} or \ref{def:ero2} of Definition 1.
\begin{definition}
	An $ n\times n $ elementary matrix in $ M_{n\times n}(\mathbb{Z}_2)$ is a matrix obtained by performing an elementary operation on $ I_n $. The elementary matrix is said to be of Type 1 or Type 2 depending on whether the elementary operation performed on $I_n$ is a Type 1 or Type 2 operation, respectively.
\end{definition}

Let us define $ S_n^{ij} $ as an elementary matrix of Type 1 that interchanges the $ i $-th row and the $ j $-th one for $ i $ and $ j $. Also define $ A_n^{ij} $ as an elementary matrix of Type 2 that adds the $ i $-th row to the $ j $-th row for $ i $ and $ j $.

When the representation of a nonsingular binary matrix is considered in the order of an elementary matrix, this representation is not unique. Also, it is difficult to determine how many equivalent representations exist for a nonsingular binary matrix. Several equivalences were proposed by Yoon and Kim~\cite{yoon2014mathematical} as Propositions 1 and 2 by way of a simple idea. The newly discovered equivalences proposed in this paper are denoted in Proposition 3. Their proof is provided in the appendix section.

\begin{proposition}[Exchange rule]
	For each $ i,j,k $ such that $ i\neq j,\ j\neq k, \text{ and } k\neq i, $ the following five exchange rules hold.
	
	\begin{enumerate*}[label=(\roman*),itemjoin={\quad}]
		\item $ A_n^{ik}A_n^{jk}=A_n^{jk}A_n^{ik} $,
		\item $ A_n^{ij}A_n^{jk}=A_n^{ik}A_n^{ij} $,
		\item $ S_n^{ij}A_n^{ik}=A_n^{jk}S_n^{ij} $,
		\item $ S_n^{ij}A_n^{ki}=A_n^{kj}S_n^{ij} $, and
		\item $ S_n^{ij}S_n^{jk}=S_n^{jk}S_n^{ik}=S_n^{ik}S_n^{ij} $.
	\end{enumerate*}
\end{proposition}

\begin{proposition}[Compaction rules]
	For each $ i,j,k $ such that $ i\neq j,\ j\neq k, \text{ and } k\neq i $, the following two exchange rules hold.
	
	\begin{enumerate*}[label=(\roman*),itemjoin={\quad}]
		\item $ A_n^{ik}A_n^{jk}A_n^{ij}=A_n^{ij} A_n^{jk} $ and
		\item $ A_n^{kj}A_n^{ki}A_n^{ij}=A_n^{ji}A_n^{kj} $.
	\end{enumerate*}
\end{proposition}

\begin{restatable}{proposition}{myprop}
	\label{prop:myprop}
	For each $ i $ and $ j $ such that $ i\neq j $, the following three rules hold.
	
	\begin{enumerate*}[label=(\roman*),itemjoin={\quad}]
		\item $ A_n^{ij}S_n^{ij}=A_n^{ji}A_n^{ij} $,
		\item $ S_n^{ij}=A_n^{ij}A_n^{ji}A_n^{ij} $, and
		\item $ \left(A_n^{ij}A_n^{ji}\right)^2=A_n^{ji}A_n^{ij} $.
	\end{enumerate*}
\end{restatable}

For example, the encodings of matrices $ P_1 $ and $ P_2 $ are as follows: let $ P_1=S_4^{12}A_4^{21}A_4^{12} $ and $ P_2=A_4^{21}S_4^{12} $. Then, calculate $ d_e\left(P_1,P_2\right) $ based on a sequence alignment between $ P_1 $ and $ P_2 $, where $ d_e $ is the edit distance and the insertion, deletion, and replacement functions have weights of one, one, and two, respectively. First, consider the original form:
\begin{align*}
P_1&=S_4^{12}A_4^{21}A_4^{12}\ -\ , \\
P_2&=\ -\ \ -\ A_4^{12}S_4^{12}.
\end{align*}
Then, $ d_e\left(P_1,P_2\right)=3 $. This allows the parents to be changed into other forms. Note that:
\begin{equation*}
P_1=S_4^{12} A_4^{21} A_4^{12} = \left(S_4^{12} A_4^{21} A_4^{12} \right) \left(A_4^{21} A_4^{21} \right) = S_4^{12} \left(A_4^{21} A_4^{12} A_4^{21}\right) A_4^{21} = S_4^{12} S_4^{21} A_4^{21} = A_4^{21}.    		
\end{equation*}
From these rules, $ d_e\left(P_1,P_2\right) = 2 $. Thus, the propositions can produce offspring that are more similar to the parents.

\subsection{Crossover} \label{sec52:crossover}
Any recombination for a variable-length string can be used as a recombination operator for the encoding and the edit distance is typically used as the distance for the variable-length string. This changes one string into another by using a minimum number of insertions, deletions, and replacements of the elementary matrix. A geometric crossover that is associated with this distance is called a homologous geometric crossover~\cite{moraglio2006geometric}.

Several general string genetic operators can be used. In the case of a string encoding of the elementary matrix, a mathematically designed genetic operator was proposed~\cite{yoon2014mathematical}. Specifically, the geometric crossover by sequence alignment is expected to be effective. Here, alignment refers to allowing the strings to stretch in order to provide a better match. A stretched string involves interleaving the symbol `—' anywhere in the string to create two stretched strings of the same length with a minimum Hamming distance. The offspring is generated by applying a uniform crossover to the aligned parents after removing the `—' symbols. Here, two offspring solutions are generated as solutions of the two parents.

The optimal alignment of the two strings is as per the Wagner-Fischer algorithm~\cite{wagner1974string}, which is a dynamic programming (DP) algorithm that computes the edit distance between two strings of characters. This algorithm has a time complexity of $ O\left(mn\right) $ and a space complexity of $ O\left(mn\right) $ when the full dynamic programming table is constructed, where $ m $ and $ n $ are the lengths of the two strings.

\subsection{Initial Population Generation, Selection, and Mutation Replacement}
An initial population is generated with a random number of random elementary matrices. The random number is generated from a normal distribution where the mean is $ 3n $ and the standard deviation is $ n $ when the problem size is $ n $. If the random number is smaller than one, it is fixed at one. The selection operator applies a tournament selection method by choosing three parents. The mutation operator applies one of three operations, namely insertion, deletion, or replacement, to each string with a 5\% probability. Furthermore, the probability that each individual will be mutated is set at 0.2. Lastly, replacement refers to replacing the parent generation with an offspring generation. The details of this process are as follows: the selection operator is used for candidates of the offspring generation. When the population of the parent generation is $ p $, then $ p $ parents are extracted by applying the selection operator $ p $ times. The probability of two parents pairing up and applying the crossover is 0.5. When the crossover is not applied, the two parents become candidates for members of the next generation, while in the opposite case, the two offspring become candidates for members of the next generation. Each candidate proceeds with a mutation probability of 0.2 and replaces the parent generation with the next generation.

\section{Experiments} \label{sec6:experiments}
\subsection{Target Problem in Binary Representation}
In this section, two problems are described for which better solutions can be obtained with an appropriate basis.
\begin{enumerate}
	\item  Variant-onemax: for the evaluation function $ F $ of the onemax problem, vector $ v $ has an evaluation value of one. Variant-onemax is defined as counting the number of ones by changing vector $ v $ from the standard basis to a certain basis $ B $. That is, in variant-onemax, $ F\left(\left[v\right]_B\right) $ becomes the evaluation value for vector $ v $.
	
	If the basis is changed for $ v $ with the nonsingular binary matrix $ \left[T\right]_{B_s}^B $, then we have $ \left[T\right]_{B_s}^B v=\left[v\right]_B $. Then, $ F\left(\left[v\right]_B\right) $ becomes a function that counts the number of 1s in $ \left[v\right]_B $. This is therefore identical to the onemax problem as a result of an appropriate change of basis in variant-onemax. Meanwhile, from $ F\left(\left[v\right]_B\right)=F\left(\left[T\right]_{B_s}^Bv\right) $, an evaluation function of variant-onemax can be generated even when a nonsingular binary matrix is given. As for the optimum solution of variant-onemax, when the problem size is $ n $, the number of ones becomes $ n $ through the change of basis, and $ n $ becomes the optimal solution.
	
	\item $ NK $-landscape: the $ NK $-landscape model consists of a string of length $ N $ and a fitness contribution is attributed to each character depending on the other $ K $ characters. These fitness contributions are often randomly chosen from a particular probability distribution. In addition, the number of hills and valleys can be adjusted by varying $ N $ and $ K $. One of the reasons why the $ NK $-landscape model is used in optimization is that it is a simple instance of an NP-hard problem.
	
\end{enumerate}

In the experiments, the GA is used to search for solutions to the above the two problems. The GA consists of tournament selection, one-point crossover, and flip mutation, and the replacement replaces all the parent generations with offspring generations. The tournament selection process chooses the best solution among three randomly selected parents, the one-point crossover combines a solution involving two offspring with the solution of two parents, while in flip mutation, each gene is flipped from zero to one or from one to zero with a probability of 0.05. The replacement method is the same as that described in Section~\ref{sec5:find_basis_GA}. In other words, in the composition of the next generation, the number of parents extracted is equal to the number in the population. Two parents are paired up with a 50\% probability that the crossover will be applied. When the crossover is not applied, the two parents become member candidates of the next generation, while in the opposite case, the two offspring become member candidates of the next generation. Each member candidate undergoes mutation with a 20\% probability that it will replace an existing parent. When the chromosome length of variant-onemax or $ NK $-landscape is $ n $, the size of the population is set to $ 4n $. Because the fitness of the optimum solution of the variant-onemax problem is $ n $, solutions of 10,000 generations have to be searched until an optimum solution has been identified. In the $ NK $-landscape, the fitness of the optimum solution is different for each $ N,\ K $, and all solutions must be searched to obtain an optimum solution. Thus, 300,000 generations must be searched to find an optimum solution for the $ NK $-landscape problem.

\subsection{Results}
The evaluation function of variant-onemax requires a nonsingular binary matrix that corresponds to a basis. For the basis of variant-onemax that has a chromosome length of $ n $, a random number of elementary matrices are generated and then are multiplied sequentially. The number of elementary matrices is generated from a normal distribution that has a mean of $ 3n $ and a standard deviation of $ n/2 $.

In the experiment, instances of variant-onemax where $ n $ was 20, 30, and 50 were generated. With the GA described in Section~\ref{sec5:find_basis_GA}, the following bases were searched for each instance: meta-GA-based basis $ B_1 $, epistasis-based basis $ B_2 $ where the sampling number was $ n^2 $, and epistasis-based basis $ B_3 $ where the sampling number was $ n^3 $.

A total of 100 independent searches were conducted for each instance, and the number of times that an optimum solution was identified was counted along with the execution time. The results for the variant-onemax experiment are shown in Table~\ref{tab:result_var}. In the table, a type of `Original' indicates that a solution instance was evaluated without a change of basis. Similarly, `Meta,' `Epistasis-sq,' and `Epistais-cu' refer to evaluating solution instances by applying $ B_1,B_2 $, and $ B_3 $, respectively, to change the basis. In addition, the box plot in Figure~\ref{fig:boxplot_var} depicts the fitness distribution of the 100 best solutions obtained by performing 100 independent searches for each instance. A fitness is a value between zero and one that can be obtained by dividing the fitness of the optimum solution. That is, a value of one on the $ y $-axis indicates the fitness of an optimum solution, while values approaching zero indicate a lower fitness. In most cases, it can be seen that the search performance of the GA is efficient with the change of basis. When $ N $ is 50, `Epistasis-cu' does not seem to improve the search performance of the GA. This was likely because the population of the GA was not evenly distributed throughout the sample population.

In Table~\ref{tab:result_var}, `Meta' found opimal solutions more frequently than the other methods. In particular, when $ n $ was 30, the 82nd most optimal solution was obtained out of 100. This indicates that the corresponding basis was appropriate. However, because the computation time for this approach was very long, it cannot be applied in practice. Note that when $ n $ is 50, it was over 2 hours. Furthermore, no difference was observed when compared to the case in which the basis was not changed. The method of evaluating the basis using the epistasis provides a good indication of when changing the basis will provide a better result. In particular, when $ n $ is 20, the number of optima found in `Original' is 30, and the numbers of optima found in `Epistasis-sq' and `Epistasis-cu' are 64 and 33, respectively. In summary, these tests confirmed that a sample size of $ n^2 $ provided good results while requiring less time than a sample size of $ n^3 $. Therefore, in terms of time and performance, a sample size of $ n^2 $ was deemed reasonable for estimating an epistasis.

\begin{table}[H]
	\centering
	\caption{Results of each of the best solutions obtained by conducting the GA experiments 100 times on an instance of the variant-onemax problem. (`\# of optima' is the number of optima found during 100 experiments, `Average' is the average of 100 best solutions, and `SD' is the standard deviation of 100 best solutions. $ Q_1, \, Q_2, $ and $ Q_3 $ are the first, second, and third quartiles, respectively. `Time' is the sum of the time to search for the basis and that for the GA experiments.)} \label{tab:result_var}
	\vspace*{0.2cm}
	\begin{threeparttable}
	\begin{tabular}{ccccccccr@{:}l}
	\toprule
	$ n $ & Type & \# of optima & Average & SD & $ Q_1 $ & $ Q_2 $ & $ Q_3 $ &  \multithead{2}{Time \\ (mm:ss)\tnote{*}} \\
	\midrule
	\multirow{4}{*}{20} & Original		& 30 & 0.945 & 0.0452 & 0.900 & 0.950 & 1.000 & 0 & 44 \\
		                & Meta			& 66 & 0.980 & 0.0302 & 0.950 & 1.000 & 1.000 & 3 & 07 \\
		                & Epistasis-sq	& 64 & 0.982 & 0.0241 & 0.950 & 1.000 & 1.000 & 1 & 01 \\
		                & Epistasis-cu	& 33 & 0.964 & 0.2760 & 0.950 & 0.950 & 1.000 & 3 & 11 \\
	\midrule
	\multirow{4}{*}{30} & Original		& 31 & 0.963 & 0.0329 & 0.930 & 0.970 & 1.000 & 1 & 09  \\
               			& Meta			& 82 & 0.993 & 0.0155 & 1.000 & 1.000 & 1.000 & 12 & 15 \\
						& Epistasis-sq	& 47 & 0.979 & 0.0216 & 0.967 & 0.967 & 1.000 & 3 & 49  \\
						& Epistasis-cu	& 40 & 0.979 & 0.0187 & 0.967 & 0.967 & 1.000 & 7 & 03  \\
	\midrule
	\multirow{4}{*}{50} & Original		& 0  & 0.931 & 0.0257 & 0.920 & 0.940 & 0.940 & 2 & 58   \\
						& Meta			& 0  & 0.939 & 0.0240 & 0.920 & 0.940 & 0.960 & 136 & 46 \\
						& Epistasis-sq	& 2	 & 0.934 & 0.0272 & 0.920 & 0.940 & 0.945 & 7 & 48   \\
						& Epistasis-cu	& 0	 & 0.927 & 0.0272 & 0.900 & 0.920 & 0.940 & 67 & 59  \\	
	\bottomrule
	\end{tabular}
	\begin{tablenotes}
		\footnotesize
		\item[*] On Intel (R) Core TM i7-6850K CPU @ 3.60GHz
	\end{tablenotes}
	\end{threeparttable}
\end{table}

\begin{figure}[H]
	\centering
	\includegraphics[width=0.8\textwidth]{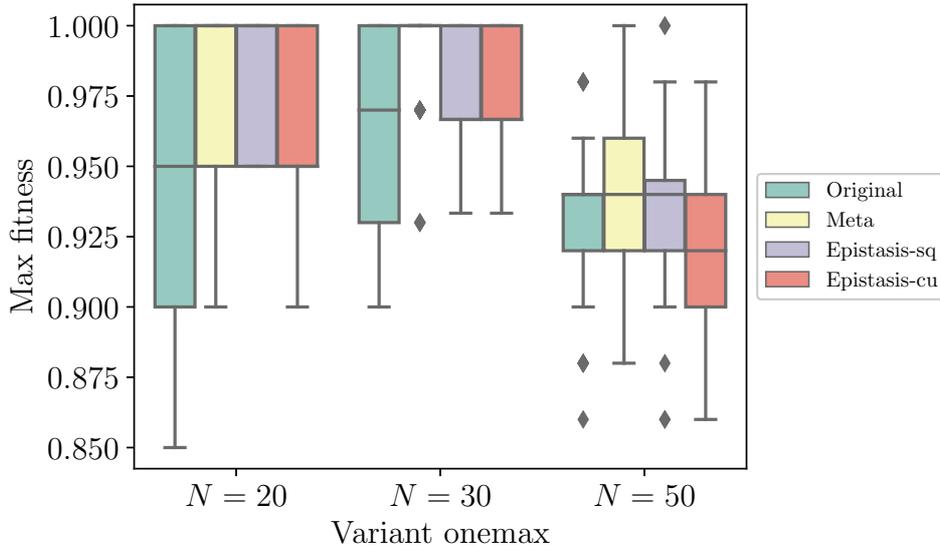}
	\caption{A box plot of each of the best solutions obtained by conducting the GA experiment 100 times on an instance of the variant-onemax problem.} \label{fig:boxplot_var}
\end{figure}

The value of $ N $ in the $ NK $-landscape experiment represents the size of the problem. In this experiment, there were $ N $ characters of zero and one and the total number of populations was $ 2^N $. The evaluation functions were randomly generated according to $ K $. In terms of the instance generation, each gene was dependent on $ K $ other genes and a value between $ \left[0,1\right] $ was assigned. The fitness of the $ NK $-landscape is based on the fitness of each gene. Therefore, the maximum and minimum fitness values, which are between zero and one, may be different for each instance. In the experiment, 100 independent searches for a solution are conducted for each instance. Table~\ref{tab:result_nk} shows the results of $ NK $-landscape experiment in which the best solution and the computation time for each of the 100 searches were compared. In the table, when the type is `Epistasis,' this indicates that a basis was obtained based on the epistasis using a sample set of size $ n^2 $, and which 100 independent searches were conducted for that instance. A box plot showing the distribution of the 100 best solutions is shown in Figure~\ref{fig:boxplot_nk}.

Upon analysis, the method of searching for the solution after changing the basis exhibited better performance than the original problem. In particular, in the box plot, it can be seen that the distribution of solutions obtained by changing the basis was more concentrated and had a higher mean. In the $ NK $-landscape, when `Meta' and `Epistasis' were compared, neither side exhibited better performance. However, it can be seen that the computation time of ‘Meta’ was about 4–30 times longer than that of `Epistasis.' Furthermore, although the `Epistasis' consumed slightly more time than the `Original,' it tended to have a more efficient evolutionary search. For these reasons, the method used to obtain the `Epistasis' results was found to be the best among the three methods evaluated.

\begin{table}[H]
	\centering
		\caption{Results of each of the best solutions obtained by conducting the GA experiments 100 times on an instance of the $ NK $-landscape problem. (`Best' is the best fitness among solutions found in 100 experiments, `Average' is the average of 100 best solutions, and `SD' is the standard deviation of 100 best solutions. $ Q_1, \, Q_2, $ and $ Q_3 $ are the first, second, and third quartiles, respectively. `Time' is the sum of the time to search for the basis and that for the GA experiments.)} \label{tab:result_nk}
	\vspace*{0.2cm}
	\begin{threeparttable}
	\begin{tabular}{ccccccccr@{:}l}
	\toprule
	$ N,\, K$ & Type & Best & Average & SD & $ Q_1 $ & $ Q_2 $ & $ Q_3 $ & \multithead{2}{Time \\ (mm:ss)\tnote{*}} \\
	\midrule
	\multirow{3}{*}{20, 3}	& Original  & 0.817 & 0.8135 & 0.0085 & 0.8170 & 0.8170 & 0.8170 & 1 & 02 \\
							& Meta	    & 0.825 & 0.8226 & 0.0057 & 0.8250 & 0.8250 & 0.8250 & 5 & 52 \\
							& Epistasis & 0.825 & 0.8200 & 0.0056 & 0.8170 & 0.8170 & 0.8250 & 1 & 32 \\
	\midrule
	\multirow{3}{*}{20, 5}	& Original  & 0.761 & 0.7449 & 0.0157 & 0.7400 & 0.7405 & 0.7610 & 1 & 03 \\
							& Meta	    & 0.761 & 0.7533 & 0.0131 & 0.7470 & 0.7610 & 0.7610 & 5 & 39 \\
							& Epistasis & 0.761 & 0.7505 & 0.0109 & 0.7460 & 0.7470 & 0.7610 & 1 & 40 \\
	\midrule
	\multirow{3}{*}{20, 10}	& Original  & 0.779 & 0.7306 & 0.0253 & 0.7020 & 0.7335 & 0.7520 & 1 & 10 \\
							& Meta      & 0.785 & 0.7572 & 0.0155 & 0.7660 & 0.7550 & 0.7660 & 7 & 13 \\
							& Epistasis & 0.785 & 0.7558 & 0.0136 & 0.7460 & 0.7530 & 0.7653 & 2 & 16 \\
	\midrule
	\multirow{3}{*}{30, 3}	& Original  & 0.776 & 0.7687 & 0.1373 & 0.7740 & 0.7760 & 0.7760 & 2 & 06 \\
							& Meta      & 0.776 & 0.7719 & 0.0109 & 0.7760 & 0.7760 & 0.7760 & 5 & 39 \\
							& Epistasis & 0.776 & 0.7718 & 0.0090 & 0.7740 & 0.7760 & 0.7760 & 1 & 40 \\
	\midrule
	\multirow{3}{*}{30, 5}	& Original  & 0.795 & 0.7725 & 0.0125 & 0.7638 & 0.7740 & 0.7870 & 2 & 06 \\
							& Meta      & 0.795 & 0.7661 & 0.0170 & 0.7540 & 0.7710 & 0.7770 & 32 & 28 \\
							& Epistasis & 0.795 & 0.7706 & 0.0136 & 0.7623 & 0.7730 & 0.7830 & 2 & 50 \\
	\midrule
	\multirow{3}{*}{30, 10}	& Original  & 0.779 & 0.7349 & 0.0181 & 0.7260 & 0.7310 & 0.7443 & 2 & 06 \\
							& Meta      & 0.805 & 0.7391 & 0.0179 & 0.7310 & 0.7370 & 0.7470 & 49 & 47 \\
							& Epistasis & 0.796 & 0.7366 & 0.0198 & 0.7220 & 0.7335 & 0.7960 & 3 & 48 \\
	\midrule
	\multirow{3}{*}{30, 20}	& Original  & 0.750 & 0.7039 & 0.0152 & 0.6938 & 0.7010 & 0.7113 & 2 & 51 \\
							& Meta      & 0.762 & 0.7181 & 0.0163 & 0.7070 & 0.7155 & 0.7243 & 49 & 47 \\
							& Epistasis & 0.770 & 0.7220 & 0.0133 & 0.7120 & 0.7200 & 0.7300 & 3 & 48 \\
	\midrule
	\multirow{3}{*}{50, 3}	& Original  & 0.776 & 0.7576 & 0.0102 & 0.7515 & 0.7590 & 0.7640 & 5 & 31 \\
							& Meta      & 0.776 & 0.7599 & 0.0119 & 0.7530 & 0.7585 & 0.7730 & 220 & 14 \\
							& Epistasis & 0.776 & 0.7578 & 0.0096 & 0.7508 & 0.7590 & 0.7630 & 6 & 34 \\
	\bottomrule
	\end{tabular}
	\begin{tablenotes}
		\footnotesize
		\item[*] On Intel (R) Core TM i7-6850K CPU @ 3.60GHz
	\end{tablenotes}
	\end{threeparttable}
\end{table}

\begin{figure}[H]
	\centering
	\includegraphics[width=0.8\textwidth]{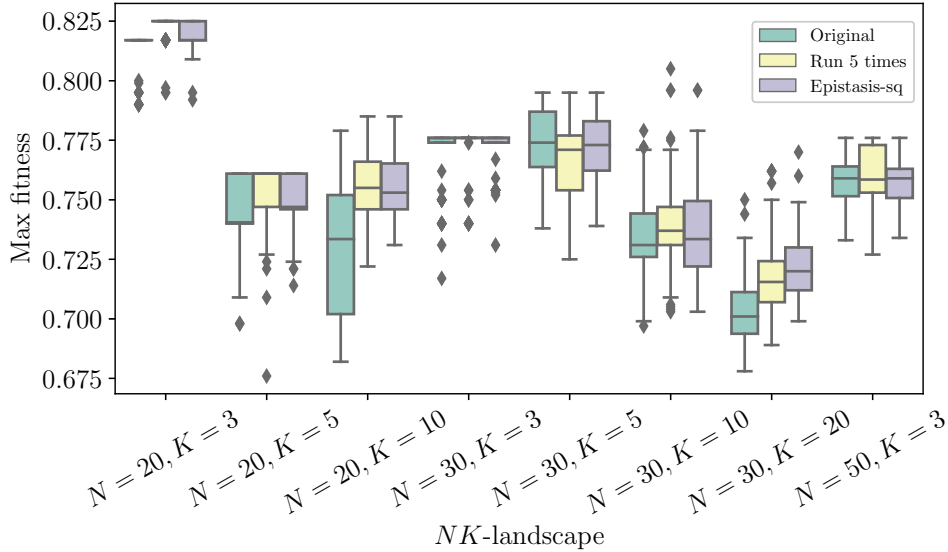}
	\caption{A box plot of each of the best solutions obtained by conducting the GA experiment 100 times on an instance of the $ NK $-landscape problem.} \label{fig:boxplot_nk}
\end{figure}

\subsection{Experimental Analysis}
The results of the above experiments confirmed that a basis obtained by estimating the epistasis improved the efficiency of searching for a solution using a GA. In this section, an analysis is performed to examine how much the basis found in the experiment reduced the epistasis. The basis was estimated in such a way that the epistasis of the sample population $ S $ was reduced. Whether the GA proposed in Section~\ref{sec5:find_basis_GA} was effective can be confirmed by comparing the epistasis of $ S $ and that of $ S^\prime $ in which the basis was changed to the one identified by the search $ S $. It is expected that the latter epistasis will be smaller.

A comparison of the epistasies between $ S $ and $ S^\prime $ in the variant-onemax and $ NK $-landscape experiments can be seen in Tables~\ref{tab:epi_var} and~\ref{tab:epi_nk}, respectively. First, in Table~\ref{tab:epi_var}, $ n $ is the chromosome length of the variant-onemax experiment. The sizes of the sample sets were $ n^2 $ and $ n^3 $, respectively; `Before' and `After' show the epistasies of $ S $ and $ S^\prime $, respectively. For every $ n $, it was confirmed that a lower epistasis value was obtained when the basis was changed. Moreover, when the sampling size was `square', the epistasis was reduced more compared to the `cubic'. Thus, there was a higher possibility that the GA would conduct a more efficient search and find a better solution. When $ n $ was 20, since there were $ 2^{20} $ solutions, the epistasis for all the solutions, not the sample sets, can be obtained. Here, it was confirmed that the epistasis was 4.50, and since the epistasis was 4.46 and 4.35 when the sampling sizes were square and cubic, respectively, this indicates that the original epistasis was accurately estimated.

In Section~\ref{sec52:crossover}, the size of the sample set $ S $ in the $ NK $-landscape experiment was $ N^2 $. Table~\ref{tab:epi_nk} shows the epistasies of $ S $ and $ S^\prime $ after the basis was changed, respectively, for the values of $ N,K $ used in the experiment. The `Before' and `After' results indicate the epistasies of $ S $ and $ S^\prime $, respectively. As in the case of the variant-onemax experiment, it was confirmed that for every $ n $, a lower value of epistasis was obtained when a change of basis was applied. When $ N $ was $ 20 $, the epistasis for all the solutions, but not the samplings, was obtained. When $ K $ was 3, 5, and 10, the epistasis was $ 3.24e^{-3},\ 3.38e^{-3},\ \text{and}\ 4.13e^{-3} $, respectively. These values are close to the respective epistases of $ S $, $ 3.17e^{-3},\ 3.16e^{-3},\ \text{and}\ 4.28e^{-3} $.

\begin{table}[H]
	\caption{Epistasis of the original and modified basis sampling in the variant-onemax problem.} \label{tab:epi_var}
	\vspace*{0.2cm}
	\centering
	\begin{threeparttable}
		\begin{tabular}{ccccc}
		\toprule
		\multirow{2.5}{*}{$ n $} & \multirow{2.5}{*}{Sampling size} & \multicolumn{3}{c}{Epistasis} \\	\cmidrule(lr){3-5}
		& & Before & After & Decrease rate (\%)\tnote{*}       \\
		\midrule
		\multirow{2}{*}{20} & square	& 4.46 & 3.23 & 27.6   \\
							& cubic		& 4.35 & 3.83 & 12.0   \\
		\midrule
		\multirow{2}{*}{30} & square	& 4.57 & 3.20 & 30.0   \\
							& cubic		& 5.00 & 3.72 & 25.6   \\
		\midrule
		\multirow{2}{*}{50} & square	& 9.27 & 7.53 & 18.8   \\
							& cubic		& 9.69 & 8.93 & { 7.8} \\
		\bottomrule
		\end{tabular}
		\begin{tablenotes}
			\footnotesize
			\item[*] $ \text{Decrease rate} = 100 \times \left( \text{Before} - \text{After} \right) / \text{Before}$
		\end{tablenotes}
	\end{threeparttable}
\end{table}

\begin{table}[H]
	\caption{Epistasis of the original and modified basis sampling in the $NK$-landscape problem.} \label{tab:epi_nk}
	\vspace*{0.2cm}
	\centering
	\begin{threeparttable}
		\begin{tabular}{r@{, }lccc}
			\toprule
			\multicolumn{2}{c}{\multirow{2.5}{*}{$ N,\, K $}} & \multicolumn{3}{c}{Epistasis} \\	\cmidrule(lr){3-5}
			\thead{} & & Before & After & Decrease rate (\%)\tnote{*} \\
			\midrule
			20 & 3	& $ 3.17e^{-3} $ & $ 2.25e^{-3} $ & 29.0    \\
			20 & 5	& $ 3.16e^{-3} $ & $ 2.90e^{-3} $ & { 8.2}  \\
			20 & 10	& $ 4.28e^{-3} $ & $ 3.82e^{-3} $ & 10.7    \\
			30 & 3	& $ 1.85e^{-3} $ & $ 1.60e^{-3} $ & 13.5    \\
			30 & 5	& $ 2.61e^{-3} $ & $ 2.37e^{-3} $ & { 9.2}  \\
			30 & 10	& $ 2.68e^{-3} $ & $ 2.39e^{-3} $ & 10.8    \\
			30 & 20	& $ 2.78e^{-3} $ & $ 2.50e^{-3} $ & 10.1    \\
			50 & 3	& $ 1.13e^{-3} $ & $ 9.32e^{-4} $ & 17.5    \\
			\bottomrule
		\end{tabular}
		\begin{tablenotes}
			\footnotesize
			\item[*] $ \text{Decrease rate} = 100 \times \left( \text{Before} - \text{After} \right) / \text{Before}$
		\end{tablenotes}
	\end{threeparttable}
\end{table}

\section{Conclusions} \label{sec7:conclusions}
In this paper, a epistasis-based evolutionary search method was proposed for estimating a basis that would simplify a particular problem. Two test problems were constructed, a basis was identified by estimating the epistasis, and after the basis was changed, the results before and after the basis change were compared. The epistasis-based basis estimation method was found to be extremely efficient compared to a meta-GA in terms of time. This was also found for the $ NK $-landscape in which the epistasis-based basis estimation method provided similar results. Thus, it is reasonable to estimate the basis by using the epistasis rather the meta-GA algorithm.

To estimate an epistasis, sample sets of size $ n^2 $ or $ n^3 $ sampling data were used. It was therefore necessary to conduct a study to find an appropriate sampling number. However, the method of finding the basis was carried out using a simple GA. In the future, a study should be conducted to identify a better basis. Also, by applying various factors in the GA or other genetic operators or by applying the method shown in appendix section, a higher quality search can be performed.

Furthermore, the experiment evaluated specific problems that could be simplified with a change of basis. In further research, it will be necessary to identify the characteristics of problems that could benefit from a change of basis. Note that the basis evaluation method is applicable to not only binary encoding, but also to $ k $-ary encoding. In addition, it can be used to evaluate any vector space in which the epistasis can be calculated.

\section*{Appendix}
We present the following lemma to prove Proposition~\ref{prop:myprop}:
\begin{lemma}
	Let $ M=(m_{ij}) $ be an $ n\times n $ binary matrix. For each $ i $ and $ j $ such that $ i\neq j $, the following four rules hold.
	\begin{enumerate}[label=(\roman*)]
		\item $ Row_i (A_n^{ij} M)=Row_i (M) $,
		\item $ Row_j (A_n^{ij} M)=Row_j (M)+Row_i (M) $,
		\item $ Row_i (S_n^{ij} M)=Row_j (M) $, and
		\item $ Row_j (S_n^{ij} M)=Row_i (M) $,
	\end{enumerate}
	where $ Row_i (M) $ is the $ i $-th row vector of matrix $ M $.
\end{lemma}
\begin{proof}
	Let $ m_i $ be the $ i $-th row vector of $ M $; that is $ M=(m_1,m_2,…,m_n )^T $. Without loss of generality, we assume that $ i<j $. Note that
	\begin{align*}
	A_n^{ij} M &= \left(\begin{matrix}
	\vdots \\
	Row_i (A_n^{ij}M) \\
	\vdots \\
	Row_j (A_n^{ij}M) \\
	\vdots
	\end{matrix}\right)
	= \left(\begin{matrix}
	\vdots \\
	m_i \\
	\vdots \\
	m_j + m_i \\
	\vdots
	\end{matrix}\right),\ \text{and} \\
	S_n^{ij} M &= \left(\begin{matrix}
	\vdots \\
	Row_i (S_n^{ij}M) \\
	\vdots \\
	Row_j (S_n^{ij}M) \\
	\vdots
	\end{matrix}\right)
	= \left(\begin{matrix}
	\vdots \\
	m_j \\
	\vdots \\
	m_i \\
	\vdots
	\end{matrix}\right).
	\end{align*}
	So, we have the following:
	\begin{enumerate}[label=(\roman*)]
		\item $ Row_i (A_n^{ij} M)=m_i=Row_i (M) $,
		\item $ Row_j (A_n^{ij} M)=m_j+m_i=Row_j (M)+Row_i (M) $,
		\item $ Row_i (S_n^{ij} M)=m_j=Row_j (M) $, and
		\item $ Row_j (S_n^{ij} M)=m_i=Row_i (M) $.
	\end{enumerate}
\end{proof}

\myprop*
\begin{proof}
	Let $ M=(m_{ij} ) $ be an $ n\times n $ binary matrix which $ m_i $ is the $ i $-th row vector; that is $ M=(m_1,m_2,…,m_n )^T $.
	\begin{enumerate}
		\item It is enough to show that the $ i $-th and $ j $-th row vectors of $ A_n^{ij} S_n^{ij} M $ are the same as those of $ A_n^{ji} A_n^{ij}  M $. Consider the left side: Using Lemma 1, we have 
		\begin{align*}
		Row_i (A_n^{ij} S_n^{ij} M)&=Row_i (S_n^{ij} M)=Row_j (M)=m_j,\ \text{and}\\
		Row_j (A_n^{ij} S_n^{ij} M)&=Row_j (S_n^{ij} M)+Row_i (S_n^{ij} M)=Row_j (M)+Row_i (M)=m_j+m_i .
		\end{align*}
		Now consider the right side:
		\begin{align*}
		Row_i (A_n^{ji} A_n^{ij} M)&=Row_i (A_n^{ij} M)+Row_j (A_n^{ij} M)=Row_i (M)+(Row_j (M)+Row_i (M))=m_j, \text{ and} \\
		Row_j (A_n^{ji} A_n^{ij} M)&=Row_j (A_n^{ij} M)=Row_i (M)+Row_i (M) = m_i+m_j .
		\end{align*}
		\item 	We know $ A_n^{ij} S_n^{ij} = A_n^{ji} A_n^{ij} $. We multiply $ A_n^{ij} $ in both sides. Then, the left side is $ A_n^{ij} A_n^{ij} S_n^{ij}=I_n S_n^{ij}=S_n^{ij} $, and so $ S_n^{ij}=A_n^{ij} A_n^{ji} A_n^{ij} $.
		\item $ S_n^{ij} = S_n^{ji} $ by the definition of $ S_n^{ij} $. Now, consider the euqation $ S_n^{ij}A_n^{ji} = S_n^{ji}A_n^{ji} $. Note that the left side
		\[S_n^{ij}A_n^{ji} = \left( A_n^{ij}A_n^{ji}A_n^{ij}\right)A_n^{ji} = \left(A_n^{ij}A_n^{ji}\right)^2,\]
		and note that the right side
		\[ S_n^{ji}A_n^{ji} = \left(A_n^{ji}A_n^{ij}A_n^{ji}\right)A_n^{ji} = \left(A_n^{ji}A_n^{ij}\right)\left(A_n^{ji}\right)^2 = A_n^{ji}A_n^{ij} . \] 
	\end{enumerate} 
\end{proof}

\section*{Data Availability}
The data used to support the findings of this study are included within the article.

\section*{Conflicts of Interest}
The authors declare that they have no conflicts of interest.

\section*{Acknowledgment}
The present research has been conducted by the Research Grant of Kwangwoon University in 2019.
This research was supported by a grant (KCG-01-2017-05) through the Disaster and Safety Management Institute funded by Korea Coast Guard of Korean government, and by Basic Science Research Program through the National Research Foundation of Korea (NRF) funded by the Ministry of Education (No. 2015R1D1A1A01060105).


\end{document}